\documentclass[letterpaper, 10 pt, conference]{ieeeconf}  % Comment this line out if you need a4paper

\usepackage{amsmath, amsfonts, amssymb, amsthm}
\usepackage{algorithm, algpseudocode}
\usepackage{mathptmx}
\usepackage{graphicx}
\usepackage{xcolor}
\usepackage{color}
\usepackage{url}
\usepackage{hyperref}
\hypersetup{pdfborder = {0 0 0}}
\usepackage{booktabs}
\usepackage{multirow}

\IEEEoverridecommandlockouts                              % This command is only needed if 
\title{\LARGE \bf
Safe Vision Language Action Models via Barrier Enhanced Flow Matching
}

\author{% <-this % stops a space
}

\author{Kasra Sinaei, Hung-Chieh Wu and Donald Ebeigbe% <-this % stops a space
\thanks{The authors are with the Department of Electrical Engineering, The Pennsylvania State University; {\tt\small kasra@psu.edu}; {\tt\small hungchieh@psu.edu}; {\tt\small ebeigbe@psu.edu}}%
}

\theoremstyle{plain} 
\newtheorem{thm}{Theorem}
\newtheorem{definition}{Definition}
\newtheorem{remark}{Remark}
\newtheorem{assumption}{Assumption}
\newtheorem{corollary}{Corollary}
\begin{document}

\maketitle
\thispagestyle{empty}
\pagestyle{empty}

%%%%%%%%%%%%%%%%%%%%%%%%%%%%%%%%%%%%%%%%%%%%%%%%%%%%%%%%%%%%%%%%%%%%%%%%%%%%%%%%
\begin{abstract}
This article presents a modular inference framework that integrates {\color{black}Flow Matching generative} models with formal Control Barrier Function (CBF) safety guarantees. Unlike existing methods that apply external safety filters to a model's final output, our approach modifies the Flow Matching denoising process within the model to inherently generate safe trajectories. By employing a smooth Log-Sum-Exponential aggregate barrier, we enforce safety over entire action chunks. {\color{black}This aggregate barrier ensures a minimal increase in computational overhead and does not alter the semantic intent of the model.} We show that, within the proposed framework, the 2-Wasserstein distance between the generated distribution and the target distribution remains bounded. Our method eliminates the need for safety-specific datasets or costly model retraining, providing a versatile solution for safe {\color{black}inference. We validate the approach on two robotic manipulation platforms and a 2D navigation benchmark, verifying} that our framework achieves reliable safety without degrading the success rate of the model.
\end{abstract}

%%%%%%%%%%%%%%%%%%%%%%%%%%%%%%%%%%%%%%%%%%%%%%%%%%%%%%%%%%%%%%%%%%%%%%%%%%%%%%%%
\section{Introduction}

\subsection{Generalist Robotic Policies}
Large behavioral models have revolutionized robotics, driven by rapid advances in transformers \cite{vaswani2017attention}, Vision-Language Models (VLM), Diffusion Policies \cite{ho2020denoising}, multi-modal datasets \cite{capuano2025robot}, and open-source robot trajectory datasets \cite{liu2023libero}. Generalist robotic policies are no longer far-fetched, as state-of-the-art methods can perform complex robotic manipulation and locomotion tasks using only inference from fine-tuned foundation models \cite{intelligence2025pi_}. Researchers have addressed challenges such as inference time, vision–action–language integration, robustness, and low success rates to develop novel architectures for robot control policies. Early efforts focused on imitation learning, leveraging large datasets tailored to specific robots and tasks. Subsequent work with action chunking transformers (ACT) \cite{zhao2023act} and diffusion policies significantly improved both performance and model architecture. More recently, Vision-Language-Action (VLA) models have been introduced to the robotics community and rapidly adopted by researchers and engineers developing robots that tackle complex real-world tasks. In addition to improved performance, the size of these newer open-source models has decreased substantially. For instance, the closed-source RT-2-X \cite{zitkovich2023rt}, one of the earlier models, has 55B parameters, whereas OpenVLA, introduced in 2024, reduced the policy size to 7B parameters \cite{kim2024openvla}, and SmolVLA \cite{shukor2025smolvla}, introduced in 2025, further reduced it to 550M parameters, making it lightweight enough that contemporary portable computers can handle both fine-tuning and inference.

Modern VLAs like Physical Intelligence $\pi_{0}$ and SmolVLA consist of two transformers with separate sets of weights, connected via a block-wise causal attention mask. The first set of transformer weights belongs to a conventional VLM, which provides strong semantic and visual understanding of the image and language inputs, while the second set of transformer weights is a smaller model called the Action Expert (AE). The action expert is essentially a subset of the model's weights that is responsible for processing the system's proprioceptive states and generating action chunks. Chen et al. {\color{black}\cite{lipman2022flow}} proposed a flow matching algorithm as an optimized denoising procedure to be used for generation; this algorithm reduces the number of denoising steps and makes real-time, high-frequency inference possible. This architecture provides an effective foundation model that takes both language tokens and camera frames as input to generate control actions for robots in real time.
\begin{figure}[t]
    \centering
    \includegraphics[width=0.99\linewidth]{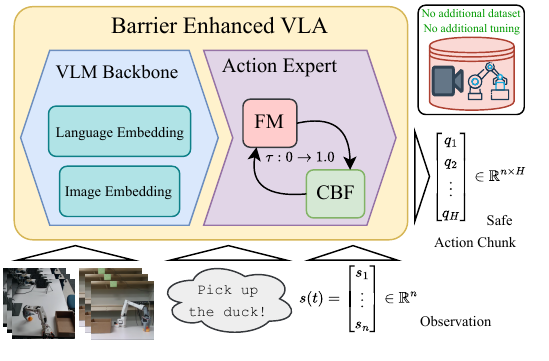}
    \caption{High-level block diagram of a safe VLA with an additional CBF block that adjusts generated action chunks for safety considerations.}
    \label{fig_summary}
\end{figure}
\subsection{Safety-critical Control Systems}
In the literature of control systems, a control design problem that has safety constraints or metrics on top of its tracking or stability requirements is considered a safety-critical control problem \cite{ames2016control}. The main focus of nonlinear system design is the stability of the states and developing a formal proof for its robust convergence to a specific target state. Safety controllers were proposed later on to address the problem of safety-critical control design. Control barrier functions (CBF) are one of the most common tools for designing safety filters in the control engineering community \cite{ames2016control}. CBFs provide a formal guarantee of safety through set invariance theorems and are compatible with different control practices such as robust control, adaptive control \cite{11456511}, and {\color{black}data-driven control \cite{zheng2025robust}}. Despite their advantages, enforcing safety via CBF has its own caveats. CBFs require a decent knowledge of the system model, and they usually result in conservative controllers. Researchers have addressed these challenges to introduce more flexibility into the CBF implementation and broaden its applications. Nowadays,  we see different variations of CBF applied to {\color{black}dynamics-free controllers} with acceptable conservative management. This makes the CBF a practical solution for enforcing safety on VLA policies.
\subsection{\color{black}Related Work}
Researchers are attempting to integrate different safety filters with {\color{black} denoising generative models utilizing diffusion and flow matching. Mizuta et al. \cite{mizuta2024cobl} first integrated CBF with a diffusion-based planning framework and developed CoBL. Safe Diffuser \cite{xiao2023safediffuser} is another example of a safe planning framework that bridges between generative models and CBFs.} Yang et al. \cite{yang2025safeflowmatcher} attempted to integrate Finite-Time Convergence CBF into the flow matching process and showed some promising results with flow matching planners. Hu et al. \cite{hu2025vlsa} proposed to filter the VLA output using a quadratic program that enforces the CBF constraint. Zhang et al. \cite{zhang2025safevla} tried to bridge the gap between some of the existing safe reinforcement learning (safe RL) methods and also extended the benchmark problems to VLAs by addressing the requirements for photorealistic simulations. They showed that a safe VLA architecture could increase both the safe operation rate and the task success rate of the model. {\color{black} Post-hoc filtering could generate out-of-distribution samples and jeopardize the quality of the trajectories \cite{romer2025demonstrations} (our comparisons in Section \ref{sec_results} yield the same conclusion on post-hoc CBF filtering). This motivates our study to develop a computationally light CBF-based framework that enables safe trajectory generation for robotic policies and does not bring safety at the cost of large distribution shifts.}
\subsection{Summary of Contributions}
VLAs, despite their capabilities in generating trajectories for complex and abstract tasks such as tabletop operation, pick and place, cleaning, and laundry folding, are not yet fully integrated with safety-critical control. In this work, we develop a method that seamlessly bridges the gap between VLA policies and classical control safety guarantees, so that the resulting agent benefits simultaneously from the complex task planning and control capabilities of VLAs and the safety guarantees of CBFs. {\color{black}Our effort is towards increasing the efficiency of previous methods and addressing the challenges they face. Our framework does not encounter the trapping issue presented with methods like Safe Diffuser~\cite{xiao2023safediffuser}, nor does it need additional guidance like Safe Flow Matcher \cite{yang2025safeflowmatcher}. Most of the state-of-the-art safe planners, like Safe Flow \cite{dai2025safeflow}, only optimize for safety; herein, we aim to involve smoothness in the problem formulation and use a single aggregate constraint to reduce the computational cost of CBF.} The main contributions of this work are summarized as follows:
\begin{enumerate}
    \item Integrating CBF with the generative policy by modifying the FM inference.
    \item {\color{black} Achieving provable safety and analyzing the conditions of guaranteed safe generation.}
    \item {\color{black}Aggregating safety constraints over action chunks and optimizing for trajectory smoothness.}
    \item {\color{black}Validating the performance of the framework in several hardware experiments and comparing against state-of-the-art methods in a 2D game.}
\end{enumerate}
%%%%%%%%%%%%%%%%%%%%%%%%%%%%%%%%%%%%%%%%%%%%%%%%%%%%%%%%%%%%%%%%%%%%%%%%%%%%%%%%
\section{Background and Preliminaries}
\subsection{Flow Matching Vision Language Action Models}
In this work, we focus on enforcing safety for modern VLAs that use Flow Matching for the generative process. These models consist of an Action Expert (AE) and a VLM backbone \cite{intelligence2025pi_}. The VLA transformer is trained on data sets containing language input, high-frequency image frames, and the robot's proprioceptive feedback (e.g., joint positions). The generative process {\color{black}aims} to approximate a flow velocity that transforms some known distribution $\pi_0$ into the distribution of dataset $\pi_1$. The deterministic continuous flow that transforms $\pi_0$ to $\pi_1$ is denoted by $\psi:[0, 1]\times \mathcal{Z} \rightarrow \mathcal{Z}$ and the corresponding vector-field is denoted by $v: [0, 1] \times \mathcal{Z} \rightarrow \mathcal{Z}$. We use $\mathcal{Z}$ for referring to the space of generated samples. The following ordinary differential equation (ODE) describes the denoising process \cite{chen2023flow}.
\begin{align}
    {\color{black}\frac{d}{d\tau}\psi (\tau, z) = v\left(\tau, \psi(\tau, z)\right);\;\; \psi(0, z) = z}\label{eq_fm}
\end{align}
Herein, $\tau \in [0, 1]$ is the denoising time where $z(\tau = 0)$ corresponds to the noisy sample drawn randomly from a normal distribution or a beta distribution and $z(\tau = 1.0)$ is the smooth denoised action chunk. {\color{black}In practice,} deriving the true vector-field $v^*(\tau, z)$ explicitly is {\color{black}challenging}, so its approximation {\color{black}$v_\theta(\tau, z) \simeq v^*(\tau, z)$} with parameter $\theta$ is used instead. 

% One of the most popular open source VLAs is the $\pi_0$ model developed by Physical Intelligence and made open source in 2025. This model showed an impressive success rate (around 80\% and higher) across a wide variety of robotic manipulation tasks. Similar to most of VLAs, this model also requires massive amounts of data and training steps for proper deployment. Another popular open source model is SmolVLA, which challenges the success rate of some larger models like OpenVLA \cite{kim2024openvla} and $\pi_0$ with only a fraction of their size.

\subsection{Control Barrier Functions}
{\color{black}Consider a nonlinear} system with state vector $x \in X \subset \mathbb{R}^n$ is modeled with the following ODE.
\begin{align}
    \frac{d}{dt} x(t) = f(x) + g(x)u \label{eq_dyn_sys}
\end{align}
This equation represents a control-affine nonlinear system where both $f(.)$ and $g(.)$ are assumed to be Lipschitz continuous functions and $u \in U \subset \mathbb{R}^m$ is the input of the system. According to the nonlinear control theorems \cite{khalil2002nonlinear}, if the control input $u$ remains Lipschitz continuous, then for any $x_0$ the system has a unique solution $x(t)$ for all $t \ge 0$. The system under control $u = k(t)$ is considered safe with respect to the safe set $\mathcal{C} \subset \mathcal{X}$ if $x(t) \in \mathcal{C}$ holds true for all $t \ge 0$. In the control barrier function literature, the safe set is defined as the 0-superlevel set of a continuously differentiable function $h: \mathbb{R}^n \rightarrow \mathbb{R}$ as shown below.
\begin{align}
    \mathcal{C} = \{x \in \mathcal{X}: h(x) \ge 0\} \label{eq_safe_set}
\end{align}
\begin{definition}\label{def_cbf}
    (Exponential CBF \cite{molnar2021model}) Any continuously differentiable function $h(.):\mathcal{X}\rightarrow \mathbb{R}$ that satisfies the following inequality {\color{black}for all $x \in \mathcal{X}$ and} some $\alpha \in \mathbb{R}_+$ is an Exponential barrier function for system (\ref{eq_dyn_sys}).
    \begin{align}\label{eq_cbf_continuous}
        \sup_{u \in \mathcal{U}} \dot h(x, u) \ge -\alpha h(x)
    \end{align}
\end{definition}
\begin{thm}\label{thm_cbf} 
    \cite{ames2016control} If $h(.)$ is a barrier function for the system (\ref{eq_dyn_sys}) and $x(0) \in \mathcal{C}$, {\color{black}then any Lipschitz continuous controller $k(x) \in \mathcal{U}$ satisfying $\dot{h}(x, k(x)) \ge -\alpha h(x)$ ensures safety of the closed-loop system with respect to safe set \eqref{eq_safe_set}.}
\end{thm}
CBFs serve as forward invariance certificates according to Theorem \ref{thm_cbf}. Control engineers use CBF to formulate quadratic programs (QP) that generate safe control inputs \cite{ames2016control}. 
\subsection{\color{black}End-to-end CBF Safety Filter}
Various modifications to CBF allow for model-free safety control of systems \cite{sinaei2025safety, molnar2021model}. In this section, we will formulate a plug-and-play safety filter that could be used as an external filter that modifies generated trajectories before sending them to the robot's low-level controller as the reference input. Consider we have a position-controlled robotic system with an asymptotically stable position controller. {\color{black}The safety-filter (\ref{eq_end2end_filter}) can filter the given input trajectory $q_d$, if the tracking controller is exponentially stable \cite{sinaei2025safety}.}
{\color{black}\begin{align}
    \color{black}q_s =& \arg\min_{q^*}\; (q^* - q_d)^T(q^* - q_d) \label{eq_end2end_filter} \\
    &\text{s.t.} \quad \nabla h \left(\frac{q^* - q}{\Delta t}\right) \ge -\alpha h(q) \nonumber
\end{align}} 
where $q_s$ is the safe position generated by the filter, $q_d$ is the desired {\color{black}next state}, {\color{black}and $q$ is the current state of the robot and $\Delta t$ is the sampling time of the controller}. The safety filter of (\ref{eq_end2end_filter}) is computationally fast and could be applied to the actions drawn from the action queue of a synchronous or asynchronous inference loop. Although effective at safety enforcement, this method is highly susceptible to the performance of the tracking controller. Some restrictions apply to the choice of parameter $\alpha$, which are critical to the safety of the system {\color{black}\cite{sinaei2025safety}}.
\subsection{Problem Formulation}
In this study, we propose and evaluate a barrier-enhanced flow matching process to be used in place of (\ref{eq_fm}) such that it enforces some safety constraint through barrier function quadratic programs (QP). The first challenge of integrating CBFs with flow matching is the lack of a system model, so we are looking into {\color{black}dynamics-free} approaches towards implementing CBF \cite{molnar2021model, sinaei2025safety}. We also need to address how to define barrier functions for action chunks. Safety description and safe set (\ref{eq_safe_set}) are regularly defined on the system's state, not a trajectory; this complicates the already challenging problem of crafting barrier functions for arbitrary safety descriptions. {\color{black} Our development is accompanied by a comparison study against some of the existing frameworks in hardware and simulation.}

Herein, generated action chunks with size $H$ are denoted by $z = \begin{bmatrix}q_1 & q_2 & \dots & q_H \end{bmatrix}$, and $q_i \in \mathbb{R}^n$ are individual joint space goal positions. Our goal is to develop a filter that enforces safety on the entire generated trajectory $z$ without destabilizing the overall denoising flow $v_\theta$.
%%%%%%%%%%%%%%%%%%%%%%%%%%%%%%%%%%%%%%%%%%%%%%%%%%%%%%%%%%%%%%%%%%%%%%%%%%%%%%%%
\section{\color{black} Main Method}
{\color{black}In this section, we first define safety for the action chunks and formulate it efficiently using a Log-sum-exponent formula that helps speeding up the safe inference (Theorem \ref{thm_lse}). Next, we develop the filtering method and demonstrate its integration with denoising models. Finally, we provide a safety guarantee (Theorem \ref{thm_fmcbf}) and show that the modified flow-matching process preserves a bounded distributional error (Theorem \ref{thm_fm_error}).}
\subsection{Barrier Enhanced Flow Matching}
The state space of \eqref{eq_end2end_filter} is defined on the robot's state $x$, not an entire trajectory (e.g. action chunk $z$). Before proceeding to the safety filter design, we need to demonstrate a systematic approach towards defining the barrier function for the action chunk that allows fast computation and efficient implementation. First, we extend the safety notion from a single action to an action chunk by the following definition. 
\begin{definition}\label{def_safe_chunk}
    (Safe Action Chunk) An action chunk $z_{t:t+H} \in \mathbb{R}^{H \times n}$ of size H is safe with respect to the safe set $\mathcal{C}$ defined via a barrier function $h(.)$, if all the states in the trajectory are safe.
    \begin{align}
        h(q_i) \ge 0 \quad \forall i \in \{ t, t+1, \dots, t+H\} \label{eq_safe_chunk} \\
        h(q_{t:t+H}) = \min_{i \in \{t, ..., t+H\}} h(q_i)
    \end{align}
\end{definition}
The main challenge with defining CBF using (\ref{eq_safe_chunk}) is that the $\min h(.)$ operator does not necessarily produce a {\color{black}continuously differentiable} barrier function \cite{molnar2023composing}. On top of the continuity problem, the barrier function composed of all states of the trajectory might end up being highly nonconvex and have a complex explicit form. Instead of defining the barrier function using the min operator, we utilize the smooth log-sum-exponential function (\ref{eq_log_exp}) to combine and smooth the resulting barrier function, which represents an approximated safe action chunk. Note that we flatten action chunk $z$ to make the dimensions of the equations consistent through QP formulation and denote the flattened $z \in \mathbb{R}^{H \times n}$ by $\breve{z}\in \mathbb R^{nH}$.
\begin{align}
    h(\breve z) = -\frac{1}{\kappa} \ln \left( \sum_{i = t}^{t+H} e^{-\kappa h(q_i)}\right) \label{eq_log_exp}
\end{align}
The smoothing parameter $\kappa$ could be any positive real value. {\color{black}Small values of $\kappa$ result in smoother approximated sets, while larger values deliver a more accurate approximation of the min(.) operator}. 
\begin{thm}\label{thm_lse}
    Consider a safe action chunk defined by \eqref{eq_safe_chunk}. Function $h(.)$ in \eqref{eq_log_exp} under-approximates the safe action chunk defined by the min operator \eqref{eq_safe_chunk} with the following bounds:
    \begin{align}
        \min_{i \in \{t, ..., t+H\}} h(q_i) - \frac{\ln H}{\kappa} \le h(\breve z) \le \min_{i \in \{t, ..., t+H\}} h(q_i) \label{eq_logsumexp_bound}
    \end{align}
\end{thm}
Proof of this theorem can be found in the appendix of \cite{molnar2023composing}, which utilizes the monotonicity property of the $\ln (.)$ function. 
\begin{remark}
    Tuning the smoothing parameter $\kappa$ to a higher value results in a smaller safety margin in the safety filter design. This is reflected in Fig. \ref{fig_kappa_comp} where the value of the barrier function is plotted in four inferences of safe VLA with different values of $\kappa$. The red region  of Fig. \ref{fig_kappa_comp} depict unsafe states, and higher positive values of $h(q)$ correspond to larger safety margins. This observation matches the under-approximation error bound $\frac{\ln H}{\kappa}$ in \eqref{eq_logsumexp_bound}.
\end{remark}
For implementation purposes, we calculate the partial derivative of (\ref{eq_log_exp}) with respect to flattened action chunk $\breve z$ by:
\begin{align}
    \frac{\partial h(\breve z)}{\partial \breve z} = \begin{bmatrix} \lambda_1 \nabla_{q}h(q_1) & \dots & \lambda_H \nabla_{q}h(q_H)\end{bmatrix}\label{eq_log_exp_deriv}
\end{align}
where we have used $\lambda_i = e^{-\kappa (h(q_i)-h(z))}$ for ease of notation. Note that $\nabla_{q}h(.) \in \mathbb{R}^{1 \times n}$ is the gradient of the barrier function with respect to the robot's state $q$. Now, with this definition of safety for action chunks (\ref{eq_safe_chunk}) and the smooth log-sum-exponent approximation formula for deriving the barrier function (\ref{eq_log_exp}), we can proceed to formulating the CBF for the flow matching.

Recall the flow matching equation (\ref{eq_fm}); in this section, we modify this equation by introducing an additional term $\delta_s \in \mathbb{R}^{H \times n}$ to the original ODE. The main purpose of this variable is to make minimal adjustments such that the action chunk is safe according to Definition \ref{def_safe_chunk}. This objective could be formulated as a QP that is constrained by the control barrier function inequality. 
\begin{align}
    {\color{black}\frac{d}{d\tau}\psi_s (z, \tau) = v(\tau, \psi(z, \tau)) + \delta_s}\label{eq_fm_enhanced}
\end{align}

Inspired by the CBF QP, we first propose the following optimization problem for finding the optimal value of $\delta_s$. 
\begin{align}
    \breve\delta_s =& \arg \min_{\delta \in \mathbb{R}^{nH}} \;\;\frac{1}{2} \| \delta \|_2^2 \label{eq_fm_cbf}\\
    & \text{s.t.}\quad \frac{\partial h}{\partial z}(\breve v_\theta( \tau, z)+\delta) \ge -\alpha h(z) \nonumber
\end{align}
Since we flattened the action chunk $z$ in the previous section, we also use a flattened equivalent vector of $\delta_s$ in the QP \eqref{eq_fm_cbf} and denote it with $\breve \delta_s \in \mathbb{R}^{nH}$. The original form could be recovered from the flattened vector.
The optimization problem \eqref{eq_fm_cbf} is feasible as long as the safe set defined by $h$ is non-empty and its solution could be obtained by applying the Karush–Kuhn–Tucker (KKT) conditions \cite{boyd2004convex}:
\begin{align}
    \delta_s = \frac{\text{ReLU} \big(-\alpha h(z) - \nabla_z h(z) v_\theta(\tau, z)\big)}{\|\nabla_z h\|^2}\label{eq_qp_sol}\nabla_z h
\end{align}
\subsection{\color{black} Distribution Shift Analysis}
Benton et al. \cite{benton2023error} studied the error bound of {\color{black}flow matching methods} and found an upper bound on Wasserstein distance under some assumptions on the $L^2$ approximation error and certain regularity conditions. This bound is dependent on the quality of the approximated {\color{black}vector-field} and its Lipschitz continuity (Theorems 1-3 of \cite{benton2023error}). Assume that our {\color{black}target distribution} is $\pi_1$, while $\pi_0$ is a known {\color{black}easy to sample} distribution (e.g., Normal or Beta). {\color{black}The random process} that transforms the initial known distribution $\pi_0$ to the target distribution $\pi_1$ is {\color{black}denoted by $Y$ and is approximated with vector-field} $v_\theta$ such that if we take a sample $Y_0 \sim \pi_0$, then the distribution of $Y_1$ is approximately $\pi_1$. {\color{black}The generated distribution by $v_\theta$} is denoted by $\hat\pi_1$. The following assumptions are required for deriving an error bound on the 2-Wasserstein distance of $\pi_1$ and $\hat\pi_1$.
\begin{assumption}\label{ass_1}
    If we denote the true flow velocity between $\pi_0$ and $\pi_1$ by $v^*(\tau, z)$, the following bound on $L^2$ approximation error exists.
    \begin{align}
        \int_0^1 \mathbb{E}[\| v_\theta(\tau, z) - v^*(\tau, z)\|]^2 d\tau \le \varepsilon^2 \label{eq_l2_error}
    \end{align}
\end{assumption}
\begin{assumption}\label{ass_2}
    For each sample $z$ and denoising time $s \in (0, 1)$, there exists unique flows $(\psi_{s,t})_{s\le t\le 1}$ and $(\psi^*_{s,t})_{s\le t\le 1}$ such that their induced {\color{black}vector-fields} $v_\theta(\tau, z)$ and $v^*(\tau, z)$ are continuously differentiable with respect to all parameters.
\end{assumption}
\begin{assumption}\label{ass_3}
    The approximated {\color{black}vector-field $v_\theta$} is continuous in both $\tau$ and $z$; it is also {\color{black}locally} Lipschitz continuous in $z$ with some Lipschitz constant $L$ for all $\tau \in (0, 1)$.
\end{assumption}
The first assumption gives us an upper bound on the $L^2$ approximation error of the $v_\theta$, which is not restrictive for VLA generative frameworks. Assumption \ref{ass_2} is required for the existence and uniqueness of the solutions to ODE \eqref{eq_fm_enhanced}. The approximator $v_\theta$ could also be trained in a manner that satisfies the local Lipschitz condition of Assumption \ref{ass_3}.
\begin{thm}[Wasserstein Error Bound \cite{benton2023error}]\label{thm_fm_error}
    Assume $\pi_0$ and $\pi_1$ are some probability distributions on $\mathbb{R}^{Hn}$; $Y_\tau$ is the {\color{black}random process transforming} $\pi_0$ with vector-field $v_\theta$, and $\hat\pi_1$ is the law of $Y_1$. Then, under Assumptions \ref{ass_1}-\ref{ass_3}, the following inequality holds on the 2-Wasserstein difference of the true and approximated distributions.
    \begin{align}
        W_2(\hat\pi_1, \pi_1) \le \varepsilon \exp\left(\int_0^1L d\tau\right) \label{eq_was_bound}
    \end{align}
\end{thm}
{\color{black}Now we show that the solution of CBF-QP \eqref{eq_fm_cbf}, shown in \eqref{eq_qp_sol},} satisfies Assumptions \ref{ass_1}-\ref{ass_3} with slightly inflated bounds such that the Wasserstein distance between {\color{black}$\hat\pi_1=\text{Law}(Y_1)$} and $\pi_1$ remains bounded. 
\begin{corollary}\label{cor_1}
    Under the assumption that the barrier function $h(z)$ is well-posed and the robot operates within a compact configuration space away from kinematic singularities, the $L^2$ error bound of the {\color{black}vector-field} $\hat{v}_\theta(\tau, z) = v_\theta(\tau, z) + \delta_s$ inflates to $\bar{\varepsilon}<\infty$, where
    $\bar{\varepsilon}^2 \le \varepsilon^2 + \int_0^1 \mathbb{E}[\|\delta_s(z, \tau)\|^2] d\tau $.
    If we denote the perturbed velocity field Lipschitz constant with $\bar{L} < \infty$, {\color{black}then $\bar{\varepsilon} \exp (\int_0^1 \bar L d\tau)$ is an upper bound on $W_2(\text{Law}(\psi_s(z, 1.0), \pi_1)$}.
\end{corollary}

\begin{proof}[Proof Sketch]
By applying the Minkowski inequality, the $L^2$ error bound of the modified flow inflates to $\overline{\varepsilon} \leq \varepsilon + \sqrt{\int_0^1 \mathbb{E}[||\delta_s(z, \tau)||^2] d\tau}$. This inflated error bound remains strictly finite under the assumption that the barrier function $h(z)$ is well-posed. Specifically, the barrier's gradient $||\nabla_z h(z)||$ must be bounded strictly away from zero whenever the safety filter is active, which prevents $\delta_s$ from diverging to infinity (division by zero). Furthermore, $\hat{v}_\theta(\tau, z)$ is spatially Lipschitz continuous (Assumption \ref{ass_3}). The exact perturbation term $\delta_s(z, \tau)$ is decoupled into the product of a scalar activation defined by $-\alpha h(z) - \nabla_z h(z) v_\theta(\tau, z)$ passed through a $\text{ReLU}$ operator, and a geometric projection vector $\nabla_z h(z) / ||\nabla_z h(z)||^2$. Assuming the barrier function is twice-continuously differentiable and the robot operates within a compact configuration space away from kinematic singularities, the scalar activation before the $\text{ReLU}$ is bounded and Lipschitz continuous. Because the $\text{ReLU}$ function is globally $1$-Lipschitz, its composition with the scalar activation remains Lipschitz. The geometric projection vector is similarly bounded and Lipschitz continuous over the compact domain. Let $U_{max}$ and $G_{max}$ represent the supremum of the activation and projection terms respectively, with corresponding Lipschitz constants $L_u$ and $L_g$. Because the product of bounded Lipschitz functions on a compact domain is itself Lipschitz, the modified flow satisfies the regularity condition with a new finite Lipschitz constant $\overline{L} \leq L + U_{max} L_g + G_{max} L_u$. Consequently, we can invoke Theorem 3 to conclude that $W_2(Law(\psi_s), \pi_1)$ is upper bounded by $\overline{\varepsilon} \exp(\int_0^1 \overline{L} d\tau)$.
\end{proof}
{\color{black}Corollary \ref{cor_1} ensures that safety perturbations do not push the final output out-of-distribution.} Benton et al. \cite{benton2023error} discussed that if the target distribution has some nice regularity conditions, the bound \eqref{eq_was_bound} could become linear in terms of the Lipschitz constant. 
\subsection{Provable Safe Denoising\color{black}}
{\color{black}We apply the safety filter only during later denoising stages ($\tau \in [\tau_s, 1]$), as early intermediate states resemble unstructured noise lacking physical meaning for kinematic barriers.}
Theorem \ref{thm_fmcbf} provides a critical intuition over the choice of $\tau_s$ (See Figure \ref{fig_kappa_comp}) {\color{black}and provides safety guarantee for barrier-enhanced flow matching.} 
\begin{figure}
    \centering
    \includegraphics[width=0.8\linewidth]{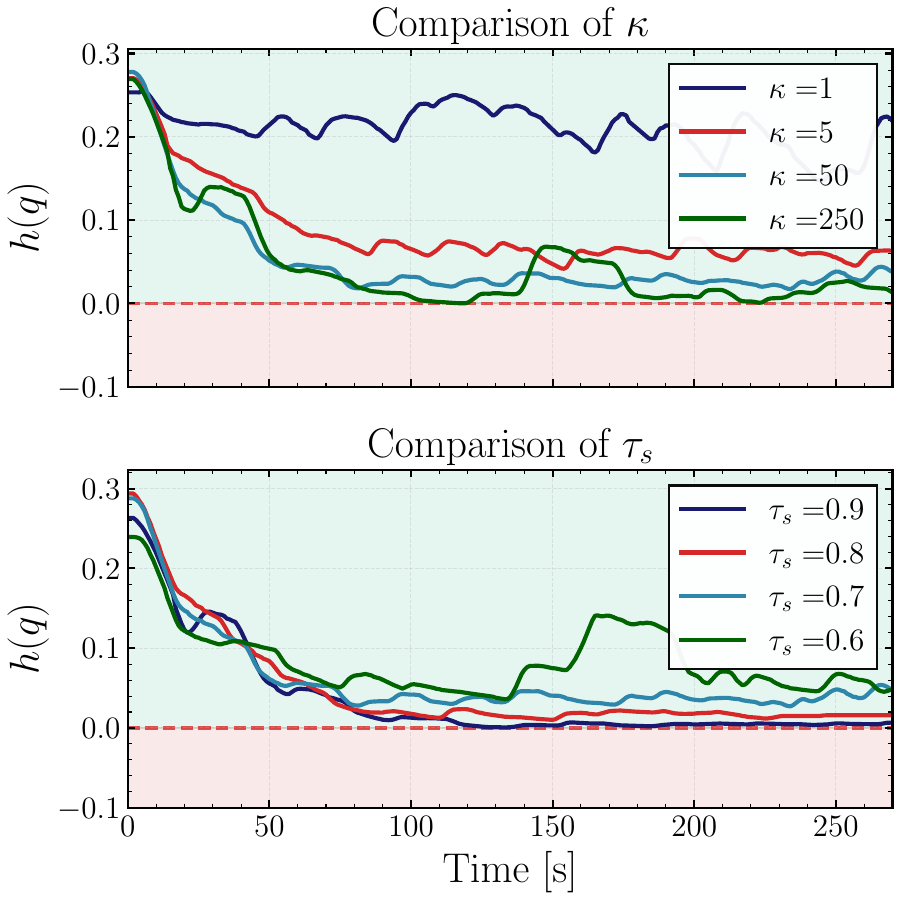}
    \caption{\color{black}Ablation study demonstrating the impact of hyper-parameters on the safety margin. Trajectories were recorded during a physical manipulation task using the SO-101 robot.}
    \label{fig_kappa_comp}
\end{figure}
\begin{thm}\label{thm_fmcbf}
    Assume the worst-case barrier decay rate of the flow $\psi(z_\tau, \tau)$ could be captured by a real positive value $M\in \mathbb{R}_+$ such that:
    \begin{align}
        \inf_{\psi(z, \tau)} \frac{dh}{d\tau}\bigg\rvert_{\psi(z, \tau)} = -M
    \end{align}
    and the barrier function value of the sampled noise is $h(\psi(z_0, \tau=0))=h_0$. If we apply CBF filtering via QP \eqref{eq_fm_cbf} for $\tau\in[\tau_s,\; 1.0]$, then the following lower bound exists for the value of the generated action chunk:
    \begin{align}
        h(\psi(z, 1.0)) \ge \left(h_0-M\tau_s\right)e^{-\alpha(1-\tau_s)}
    \end{align}
    Provided that the sampled noise is safe ($h_0 > 0$) and denoising starts such that $\tau_s \le \frac{h_0}{M}$ then \eqref{eq_fm_cbf} enforces safe flows. 
\end{thm}
\begin{proof}
    Let us assume the worst-case scenario for the first part of the flow matching, where the CBF filter is not active by setting $\frac{dh}{d\tau}=-M$.
    \begin{align*}
        h(\psi(z_{\tau_s}, \tau_s)) &= h_0 + \int_0^{\tau_s} \frac{dh}{d\tau}\bigg\rvert_{\psi(z, \tau)} d\tau \ge h_0 + \int_0^{\tau_s} -M d\tau \\
        &\ge h_0 - M \tau_s
    \end{align*}
    For the second part of the flow matching, where the filter is active, {\color{black}the QP enforces the differential CBF inequality $\frac{dh}{d\tau}\ge-\alpha h(\psi(z), \tau)$, yielding the following lower bound:}
\begin{align*}
    h(\psi(z, 1.0)) &\ge h(\psi(z_{\tau_s}, \tau_s)) e^{-\alpha(1-\tau_s)} \\
    & \geq \left(h_0-M\tau_s\right)e^{-\alpha(1-\tau_s)}
\end{align*}
    where we have substituted the lower bound of $h(\psi(z_{\tau_s}, \tau_s))$ in the first inequality. According to the theorem statement $\tau_s \le \frac{h_0}{M}$ and $h_0 \ge 0$ so the inequality reduces to $h(\psi(z, 1.0)) > 0$ and the generative process is safe.
\end{proof}
\begin{remark}
    % Filtering the generative process with \eqref{eq_fm_cbf} makes the safe set attractive, meaning that even if the initial state is unsafe, the filter forces it to approach the boundary of the safe set exponentially. To achieve the conditions of Theorem \ref{thm_fmcbf} in practice, we suggest an empirical procedure. Increasing the value of $\alpha$ and reducing $\tau_s$ results in better safety margins.
    Filtering the generative process with \eqref{eq_fm_cbf} makes the safe 
    set attractive; {\color{black}even when $h_0 - M\tau_s < 0$, the filter exponentially drives the trajectory toward the safe set boundary. The condition $\tau_s \le h_0/M$ in Theorem \ref{thm_fmcbf} is equivalent to the empirical guideline of reducing $\tau_s$ and increasing $\alpha$; both relax the sufficient 
    condition.}
\end{remark}

\begin{algorithm}
\caption{Barrier-Enhanced Flow Matching Inference}
\label{alg_safe_vla}
\begin{algorithmic}[1]
\State \textbf{Input:} Image frames $\mathbf{I}$, language tokens $\mathbf{L}$,  $\kappa$, $\alpha$, $\tau_s$, barrier parameters (e.g. $\{\vec{n}, d\}$ or $\{p_{center}, r\}$)
\State Sample $z_{0.0} \sim Beta(\tau = 0; 1.5, 1)$
\For{$\tau = 0$ \textbf{to} $1.0$ \textbf{with step} $\Delta \tau$} \Comment{Denoising loop }
    \State $\breve{v}_\theta \gets \text{VLA}(\mathbf{I}, \mathbf{L}, z_\tau)$ 
    \If{$\tau \ge \tau_s$} 
        \State Compute $h(q_i)$ and $\nabla_q h(q_i);\ \forall i \in \{1, \dots, H\}$ 
        \State $h(z) \gets -\frac{1}{\kappa} \ln \left( \sum_{i=1}^{H} e^{-\kappa h(q_i)} \right)$ 
        \State $\lambda_i \gets e^{-\kappa(h(q_i) - h(z))}$ 
        \State $\frac{\partial h}{\partial \breve{z}} \gets [\lambda_1 \nabla_q h(q_1), \dots, \lambda_H \nabla_q h(q_H)]$ 
        \State $\breve{\delta}_s \gets \arg \min_{\delta} \frac{1}{2} \|\delta\|_2^2$ \Comment{Solve QP \eqref{eq_fm_cbf_const}}
        % \State \quad \textbf{s.t.} $\frac{\partial h}{\partial \breve{z}}(\breve{v}_\theta + \delta) \ge -\alpha h(z)$ 
    \Else
        \State $\breve{\delta}_s \gets \mathbf{0}$
    \EndIf
    \State $z_{\tau + \Delta \tau} \gets z_\tau + (\breve{v}_\theta + \breve{\delta}_s)\Delta \tau$ 
\EndFor
\State \Return $z_{1.0}$
\end{algorithmic}
\end{algorithm}
{\color{black}Enforcing} velocity limits to the generated trajectory is possible by adding more constraints to the QP \eqref{eq_fm_cbf}. This trajectory will be passed to the low-level controllers of the robot for execution and real-time control, so by limiting the velocity on the generated trajectory $z(\tau = 1.0)$, we cannot guarantee the joint angular velocity at the hardware level. However, this will significantly increase the quality and smoothness of the generated action chunks. To formulate this constraint, we use the sparse matrix $D \in \mathbb{R}^{(H-1)n \times Hn}$ that facilitates the derivation of finite difference joint velocities within a flattened action chunk:
\begin{align}
    \underbrace{
    \frac{1}{\Delta t} \begin{bmatrix} 
    -I_{n} & I_{n} & 0  & \dots & 0 \\
    0 & -I_{n} & I_{n}  & \dots & 0 \\
    \vdots & \vdots & \ddots & \ddots & \vdots \\
    0 & 0  & \dots & -I_{n} & I_{n} 
    \end{bmatrix}}_{D} (\breve z + \breve \delta) = \begin{bmatrix}
        \dot q_2 \\ \dot q_3 \\ \vdots \\ \dot q_H
    \end{bmatrix} \label{eq_joint_ang}
    % \underbrace{\left(\begin{bmatrix}
    %     q_1 \\ q_2 \\ \vdots \\ q_H
    % \end{bmatrix} + \breve \delta\right)}_{\breve{\delta}_s}
\end{align}
The identity matrix of size $n \times n$ is denoted by $I_n$. Note that the $\Delta t$, multiplied by the block diagonal matrix $D$, is the physical sampling time of the robot, which is different from the sampling time used in denoising $d\tau$. This parameter is inversely related to the frame rate of the VLA. We can use the sparse matrix $D$ to also add a smoothing objective to the cost function of \eqref{eq_fm_cbf} such that it minimizes the difference between consecutive actions. The final form of barrier-enhanced flow matching QP is shown below.
\begin{align}
    \breve\delta_s = \arg & \min_{\delta \in \mathbb{R}^{nH}} \;\;\frac{1}{2} \| \delta \|_2^2 + \lambda \|D \big(\breve z+(\breve v_\theta + \delta)\Delta \tau\big) \|^2 \label{eq_fm_cbf_const}\\
    \text{s.t.}& \quad \frac{\partial h}{\partial z}(\breve v_\theta+\delta) \ge -\alpha h(z) \nonumber \\
    &{\color{black} v_{\min} \le D\big(\breve z + (v_\theta+\delta)\Delta \tau\big) \le  v_{\max}} \nonumber
\end{align}
We use the weighted sum of the two quadratic terms in \eqref{eq_fm_cbf_const} with parameter $\lambda$, which is a tunable positive real number. The first term of the cost function minimizes the safety filter adjustments to the flow based on the activation of the CBF constraint, while the second term improves the smoothness of the generated action chunk. Fig. \ref{fig_vel_limit} shows an example of successful velocity limit enforcement with the additional constraints. Note that QP \eqref{eq_fm_cbf_const} is solvable in real-time and does not cause computation burden during inference. 
\begin{figure}
    \centering
    \includegraphics[width=0.8\linewidth]{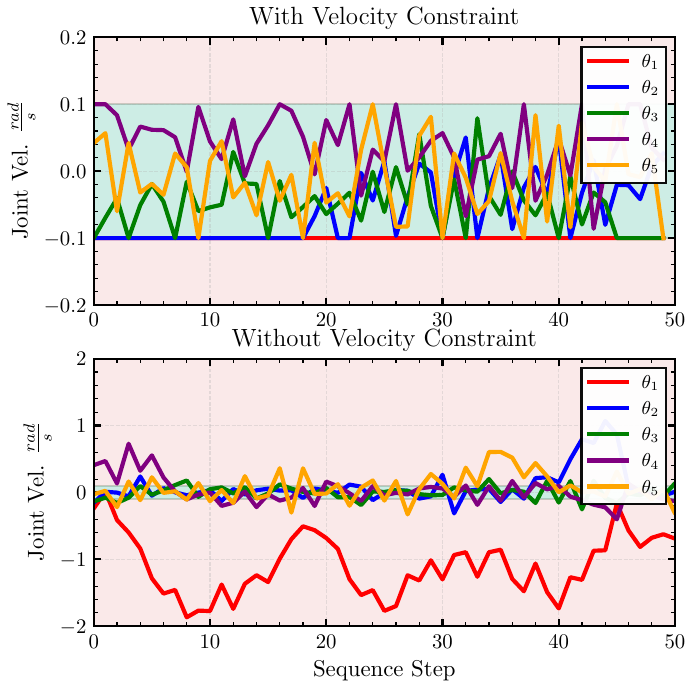}
    \caption{Effective velocity constraint enforcement via CBF QP within one action chunk. (sampled from one of the FM-CBF trials reported in Table \ref{table:results}) The velocity limit is defined as a box constraint ($\pm 0.1 \frac{rad}{s}$). }
    \label{fig_vel_limit}
\end{figure}
{\color{black}Majority of the generative models used by roboticist are expert in planning reference position and/or orientation of the system, as a result of this we focus on first order safety descriptions and their corresponding barrier function candidates. Two examples for manipulation tasks are:}

\textit{(A) Wall Barrier:} We can use this barrier function to restrict the placement of one frame (mostly end effector) from passing a 3-dimensional plane. For instance, we can use this barrier function to stop the robot from pushing its tool against its surrounding surfaces, like a table or walls. Assume the normal vector of the safe set boundary is a plane with normal vector $n$, having an offset of $d$ from the origin. The following barrier function can be used for keeping a specific robot frame on one side of the plane (e.g. Fig. \ref{fig_wall_exp}):
\begin{align}
    h(q) &= <\vec{n}, \;p_{\text{ee}}> - d \label{eq_wall_barrier} \\
    \nabla_q h(q) &= <\vec{n},\; J(q)>
\end{align}
where $<>$ denotes the vector dot product, $J(q)$ is the Jacobian of the arm for configuration $q$, and $p_{\text{ee}}$ is the Cartesian position of the end effector.

\textit{(B) Spherical Barrier:} If the user wants to keep one of the robot's frames outside the spherical region, we can formulate a barrier function based on the frame's proximity to the region's center ($p_{\text{center}}$) by the following barier function:
\begin{align}
    h(q) &= \| p_{\text{ee}}(q) - p_{\text{center}} \|_2^2 - r^2 \label{eq_spherical_barrier} \\
    \color{black}\nabla_q h(q) &= 2 (p_{\text{ee}}(q) - p_{\text{center}})^T J(q)
\end{align}
An example of this barrier is shown in Fig. \ref{fig_double_duck}. 
These barrier function candidates are mentioned as an example; the framework is compatible with other safety descriptions tailored for a wide variety of applications, such as multi-agent robotic systems \cite{multi_agent}, mobile robots \cite{molnar2021model}, and autonomous vehicles \cite{av_cite} if they comply with assumptions of Theorem \ref{thm_fm_error}.
\begin{remark}
    The aforementioned barrier functions and their time derivative could be evaluated for practical implementations without knowledge of the system's dynamics, since they are purely kinematic functions and concern the dimensions of the links and robot configuration only.
\end{remark}
%%%%%%%%%%%%%%%%%%%%%%%%%%%%%%%%%%%%%%%%%%%%%%%%%%%%%%%%%%%%%%%%%%%%%%%%%%%%%%%%
\section{Experimental Results}\label{sec_results}
\begin{figure}
    \centering
    \includegraphics[width=0.9\linewidth]{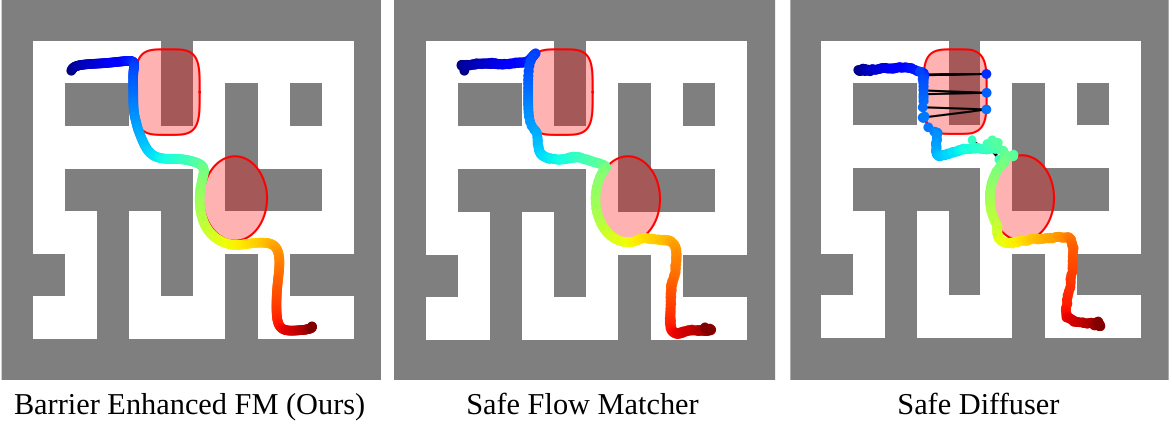}
    \caption{Sample trajectories generated by different methods for Maze 2D}
    \label{fig_maze}
\end{figure}
{\color{black}\subsection{Maze 2D with Denoising Probabilistic Models}
We evaluate our method against safe Flow Matcher \cite{yang2025safeflowmatcher} and Safe Diffuser \cite{xiao2023safediffuser} on a 2D Maze game with two unseen obstacles for safety purposes. Planning quality is measured via barrier safety (BS for two obstacles), per-plan compute time, trajectory curvature ($\kappa$), and trajectory acceleration; success rate is omitted since all methods reach 100\% over 100 trials. Table~\ref{table_maze2d} reports results with our method achieving the best trade-off across all three metrics. We used open source implementations of these methods, and used same hardware with equal computing power. The QP solver used in all inferences is the \textit{qpth} from publicly available python library. Aggregating the min operator in one constraint using \eqref{eq_log_exp} reduced the solver time; the smoothness and minimal jerk is the outcome of velocity bound and smoothing term in \eqref{eq_fm_cbf_const}.}
\begin{table}[t]
\centering
\caption{\color{black}Maze2D results (100 trials each, using qpth solver). }
\label{table_maze2d}
\setlength{\tabcolsep}{2.15pt}
\renewcommand{\arraystretch}{1.1}
\footnotesize
{\color{black}\begin{tabular}{l cccc cc}
\toprule
\textbf{Method} & \textbf{BS1} & \textbf{BS2} & \textbf{Trap } & \textbf{Time } & \textbf{$\kappa$} & \textbf{Accel.} \\
 & (${\geq}0$) & (${\geq}0$) & (\%) $\downarrow$ & (s) $\downarrow$ & $\downarrow$ & $\downarrow$ \\
\midrule
FM (no safety)                                & -0.762  & -0.938  & 0  & 1.38   & 97.7$\pm$1.5 & 151.3$\pm$4.8 \\
SafeDiffuser~\cite{xiao2023safediffuser}      & -0.003  & -0.003  & 69  & 14.41 & 68.2$\pm$90.3 & 124.5$\pm$34.2 \\
SafeFM~\cite{yang2025safeflowmatcher}         & -0.3031  & 0.003  & 12  & 14.14 & 75.9$\pm$2.9 & 195.9$\pm$22.0 \\
\textbf{CBF-FM (Ours)}                        & 0.109    & 0.046  & \textbf{0} & \textbf{10.65} & \textbf{7.2$\pm$0.4}  & \textbf{3.3$\pm$0.1} \\
\bottomrule
\end{tabular}}
\end{table}
\subsection{Manipulation Tasks with VLA}
We fine-tuned the $\pi_0$ pretrained model for a pick and place task using the SO-101 5 degrees of freedom (DoF) robotic arm with a two-finger gripper. A dataset of roughly 150 teleoperated trials, totaling 84k frames at 30 FPS, was used for the training process. {\color{black}The dataset did not involve any specific safety labeling. We also fine-tuned the model for a Quanser QArm robot which is a 4 DoF manipulator with a larger gripper on a dataset of 180k frames. To ensure the barrier gradients remain bounded away from zero (cf. Corollary 1), all objects are positioned strictly within the dexterous, singularity-free workspace of both manipulators.}
\begin{figure}
    \centering
    \includegraphics[width=0.9\columnwidth]{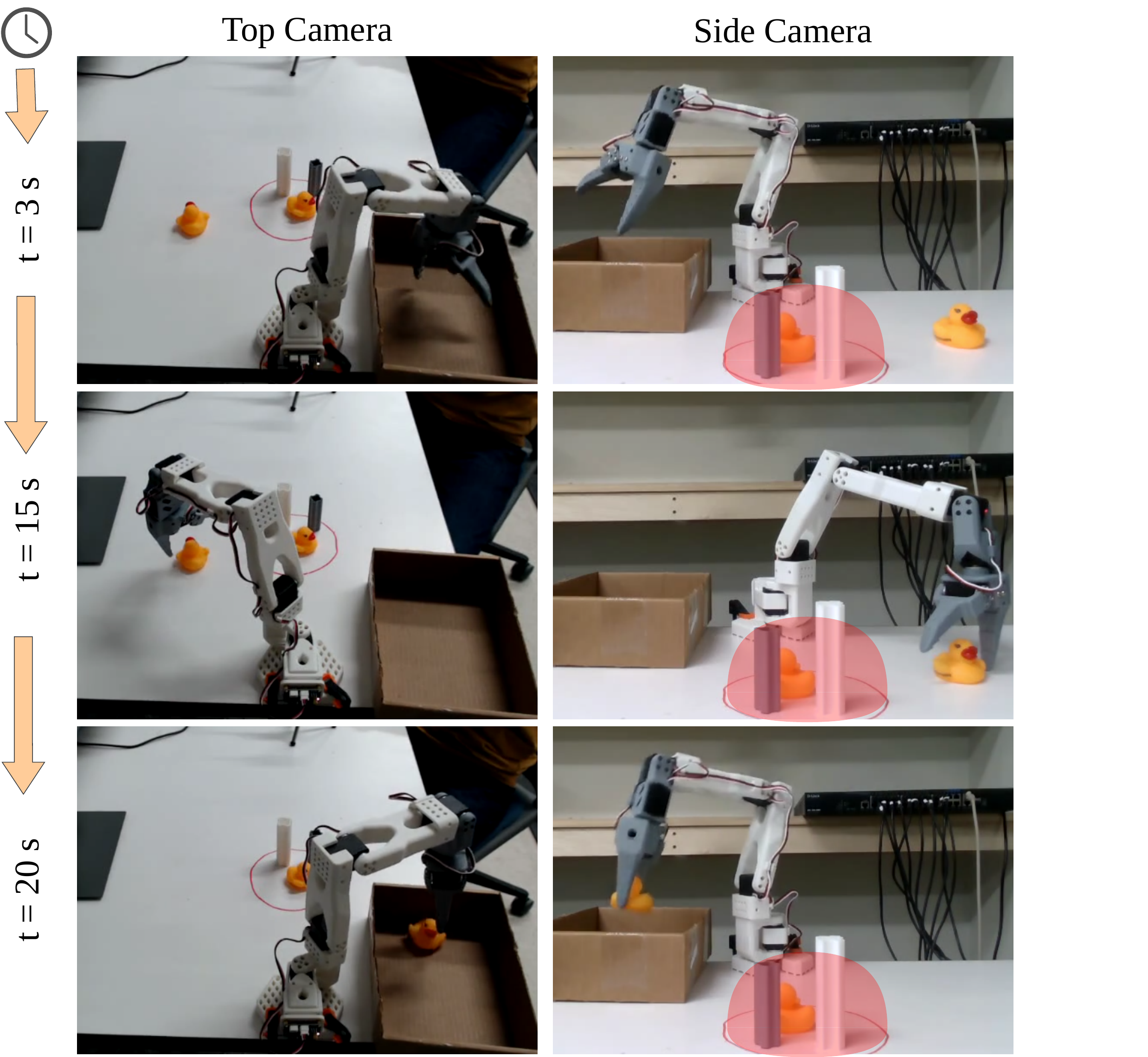}
    \caption{Rollout of safe VLA with CBF-Enhanced flow matching in an environment with obstacles. The VLA only picks up the duck that is inside the safe region and hesitates to grasp the object inside the collision zone (red sphere). {\color{black}Videos are available at \url{https://kassra-sinaei.github.io/safe-vla-webpage/}}}
    \label{fig_double_duck}
\end{figure}
Our proposed inference filter presented in Algorithm \ref{alg_safe_vla} is implemented via modification of the open source $\pi_0$ model from \cite{cadene2024lerobot}. Safety rate and success rate of the framework are measured simultaneously in experiments where {\color{black}objects are} placed inside and outside of the safe region during each trial. {\color{black}This scenario challenges safety filter and modified model verifying our safety filter does not deteriorate performance of the base VLA.} During each of the rollouts, we did one of the following:
\begin{itemize}
    \item The object is initially in the unsafe region and will be moved outside after a while. Ideally, the VLA is expected to hesitate grasping it when it is not safe and complete the task when it is moved to the safe region.
    \item {\color{black}There are multiple objects on the table to interact with. The robot should only manipulate the ones in the safe set and keeps its end-effector frame in the safe region.}
\end{itemize}
\begin{remark}
    If the robot manages to execute the task as instructed in the language input, the trial is counted as successful regardless of the collision and safety violations. Since the obstacles used in the experiment are not rigid, collisions are not fatal to the robot's health and will not terminate the experiment. This allows better evaluation of the proposed framework because the safety and success are independently evaluated in each trial, and the results are reported in Table \ref{table:results}.
\end{remark}
{\color{black}To verify the reliability and efficacy of the framework we repeated the manipulation experiment on both robots and compared the baseline VLA performance against post-hoc filter \eqref{eq_end2end_filter} and our framework (Algorithm \ref{alg_safe_vla}). Each method is repeated at least 20 times and} the measured safety rate and success rates are reported in Table \ref{table:results}. We can see that the barrier-enhanced flow matching is reliably generating safe trajectories without degrading the success rate of the base model. Additionally, the smoothing cost and joint velocity constraints in the QP \eqref{eq_fm_cbf_const} yield noticeably smoother, less jerky motions.
% Experiment videos are available at \url{https://safe-vla-fmcbf.netlify.app}.
\begin{figure}
    \centering
    \includegraphics[width=0.95\columnwidth]{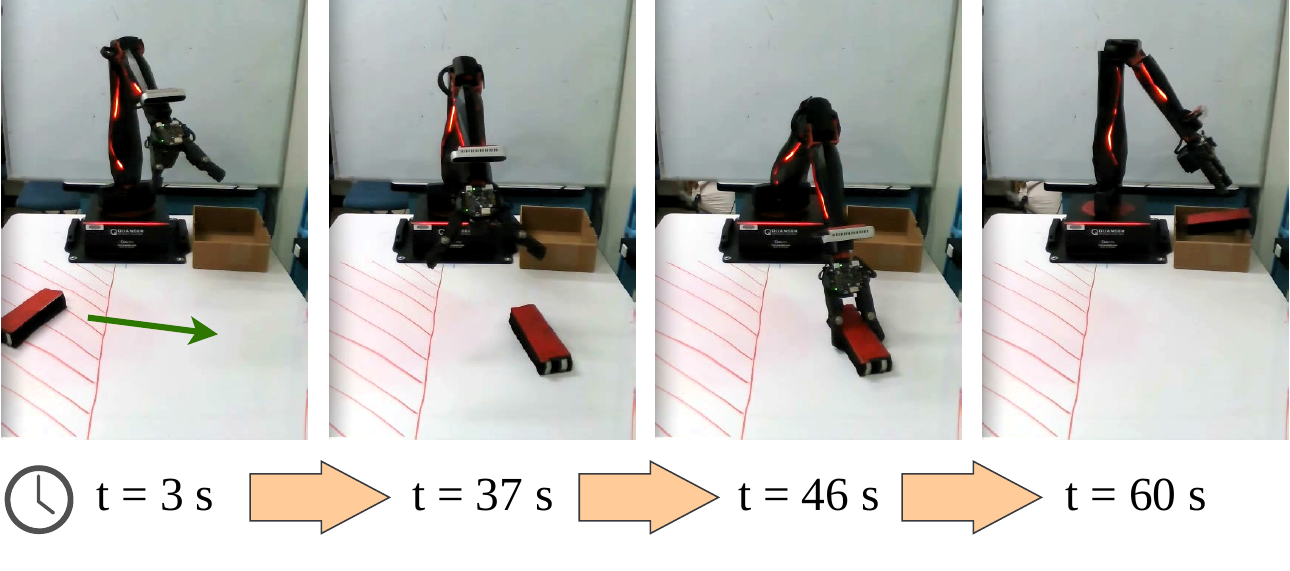}
    \caption{\color{black}Rollout of safe VLA with wall barriers performing table cleanup task. Barrier-enhanced FM does not allow the gripper to enter unsafe region. }
    \label{fig_wall_exp}
\end{figure}
\begin{table}[t]
\centering
\caption{\color{black}Hardware manipulation results across two robotic platforms.}
\label{table:results}
\setlength{\tabcolsep}{3.0pt}
\renewcommand{\arraystretch}{1.1}
\color{black}
\begin{tabular}{ll ccc}
\toprule
\textbf{Platform} & \textbf{Metric} & \textbf{No Filter} & \textbf{E2E-CBF} & \textbf{CBF-FM (Ours)} \\
\midrule
\multirow{2}{*}{\shortstack[l]{SO-101 (5-DoF)\\\textit{Pick-and-Place}}}
    & Safety (\%)  & 15.0 & 68.2 & \textbf{100.0} \\
    & Success (\%) & 75.0 & 68.2 & \textbf{77.4}  \\
\cmidrule(lr){1-5}
\multirow{2}{*}{\shortstack[l]{QArm (4-DoF)\\\textit{Table Cleanup}}}
    & Safety (\%)  & 0.00 & 70.0 & \textbf{100.0}  \\
    & Success (\%) & 100.0 & 75.0 & \textbf{100.0}  \\
\bottomrule
\end{tabular}
\end{table}
%%%%%%%%%%%%%%%%%%%%%%%%%%%%%%%%%%%%%%%%%%%%%%%%%%%%%%%%%%%%%%%%%%%%%%%%%%%%%%%%
\section{Conclusions}
{\color{black}We presented barrier-enhanced flow matching, a framework integrating control barrier functions (CBFs) with Vision-Language-Action (VLA) models. This approach enforces safety over action chunks without degrading the baseline VLA’s task completion rate. By utilizing an efficient Log-Sum-Exp QP formulation with explicit velocity limits, our method achieves faster computation times and superior planning quality (smoother trajectories with lower acceleration and curvature penalties) compared to existing generative safety filters. Crucially, it endows pre-trained VLAs with modular safety guarantees without requiring retraining or safety-specific datasets. Future work will extend this framework to handle highly dynamic obstacles and abstract safety descriptions beyond spatial collisions, further streamlining the safe deployment of generalist robotic policies.}
% \addtolength{\textheight}{-12cm}   

% \section*{ACKNOWLEDGMENT}
% Add acknowledgments here if exists any.

%%%%%%%%%%%%%%%%%%%%%%%%%%%%%%%%%%%%%%%%%%%%%%%%%%%%%%%%%%%%%%%%%%%%%%%%%%%%%%%%

\bibliographystyle{IEEEtran} 
\bibliography{references}

\end{document}